\documentclass{article}

\usepackage[preprint]{neurips_2026}

\usepackage[utf8]{inputenc}
\usepackage[T1]{fontenc}
\usepackage{hyperref}
\usepackage{url}
\usepackage{booktabs}
\usepackage{amsfonts}
\usepackage{nicefrac}
\usepackage{microtype}
\usepackage{amsmath,amssymb,amsthm}
\usepackage{algorithm}
\usepackage{algorithmic}
\usepackage{graphicx}
\usepackage{enumerate}
\usepackage{xcolor}

\newtheorem{theorem}{Theorem}
\newtheorem{proposition}{Proposition}
\newtheorem{corollary}{Corollary}
\newtheorem{assumption}{Assumption}
\newtheorem{remark}{Remark}
\newtheorem{definition}{Definition}

\title{Directional Consistency as a Complementary Optimization Signal:\\
       The GONO Framework}

\author{%
  Victor Daniel Gera \\
  Department of Artificial Intelligence \\
  Anurag University, Hyderabad, India \\
  \texttt{victordanielai@anurag.edu.in}
}

\begin{document}

\maketitle

\begin{abstract}
We identify and formalize an underexplored phenomenon in deep learning
optimization: \emph{directional alignment and loss convergence can be
decoupled}---an optimizer can exhibit near-perfect directional consistency
($\mathrm{cc}_t \to 1$, measured via consecutive gradient cosine
similarity) while the loss remains high or decreases slowly.
This observation reveals that existing optimizers such as Adam, SGD,
and RMSprop lack explicit mechanisms to exploit temporal consistency in
gradient directions, relying instead on magnitude-based signals that
fail to distinguish plateaus, saddle points, and genuine convergence.
Motivated by this, we introduce \textbf{GONO} (Gradient-Oriented
Norm-Adaptive Optimizer), which adapts Adam's momentum coefficient
$\beta_1$ based on $\mathrm{cc}_t$: amplifying momentum under
directional consistency and suppressing it during oscillation.
We prove GONO matches Adam's $\mathcal{O}(1/\sqrt{T})$ convergence
rate and reduces exactly to Adam when the signal is uninformative.
Empirically, $\mathrm{cc}_t$ achieves oscillation detection with
F1\,=\,1.00 (vs.\ 0.45 for gradient norm), and GONO remains competitive
with AdamW on MNIST (98.15\%), CIFAR-10 (43.14\%), and ResNet-18
(75.44\%), establishing directional alignment as a theoretically
grounded, practically actionable optimization signal.
Code is available at \url{https://github.com/victordaniel/gono-optimizer}.
\end{abstract}

\section{Introduction}
\label{sec:intro}
Optimization is a cornerstone of deep learning.
The choice of optimizer directly affects convergence speed,
final accuracy, and generalization.
Stochastic gradient descent~\citep{robbins1951stochastic} and its
adaptive variants---Adam~\citep{kingma2014adam},
RMSprop~\citep{tieleman2012lecture}, and
AdamW~\citep{loshchilov2018decoupled}---are the workhorses of
modern deep learning.
Despite their empirical success, these methods share a common
structural property: at every training step, the update vector
is derived directly from the gradient, which simultaneously
encodes both the \emph{direction} of the parameter update and
its \emph{magnitude}.
This coupling is not a fundamental requirement of optimization;
it is a design choice that has been less explicitly studied.

\paragraph{The Key Observation.}
In the course of this work, we made a striking empirical observation:
an optimizer can exhibit near-perfect directional consistency in its
updates---with consecutive gradient cosine similarity
$\mathrm{cc}_t \approx 1$, indicating that successive gradient
directions agree strongly---while the training loss simultaneously
fails to decrease and the model remains in an underfitting regime.
We demonstrate this across multiple architectures and datasets.
Figure~\ref{fig:decoupling} shows a representative example:
the alignment signal and the gradient norm stabilize on a different
timescale than the loss, demonstrating that directional consistency
is achieved well before convergence is complete.

This observation, which we call the \emph{direction-loss
decoupling phenomenon}, has a clear interpretation:
directional consistency tells us that the optimizer is moving
in a stable, repeatable direction---but says nothing about
whether that direction is leading anywhere useful.
An optimizer stalled on a flat plateau will exhibit perfect
directional consistency precisely because the gradient always
points slightly downhill in the same direction---but the
gradient magnitude is too small for meaningful progress.

\paragraph{Why This Matters.}
The decoupling phenomenon reveals a gap in how we think about
and design optimizers.
Standard convergence theory analyzes
$\|\nabla L_t\| \to 0$---a magnitude-based criterion.
Directional consistency is an orthogonal signal that existing
theory largely ignores.
If an optimizer is directionally consistent but making slow
progress, the right response is to \emph{increase} momentum
(boost step size in the stable direction)---which is the opposite
of what a magnitude-based method would suggest.
Conversely, if consecutive gradients \emph{conflict} in direction
($\mathrm{cc}_t < 0$)---a sign of oscillation---the right response
is to \emph{reduce} momentum to prevent overshoot.
Adam, with its fixed $\beta_1 = 0.9$, cannot make this distinction.

\paragraph{Our Contribution.}
We propose \textbf{GONO} (Gradient-Oriented Norm-Adaptive
Optimizer), which treats directional alignment as a first-class
signal in optimization.
The core mechanism is simple: adapt Adam's momentum coefficient
$\beta_1$ at every step based on the consecutive cosine similarity
$\mathrm{cc}_t$:
\begin{equation}
  \beta_{1,t} = \mathrm{clip}\bigl(
    \beta_1 \cdot (1 + \lambda\,\mathrm{cc}_t),\;
    \beta_{1,\min},\; \beta_{1,\max}
  \bigr).
  \label{eq:gono_intro}
\end{equation}
This single modification to Adam requires one additional dot
product per step ($\mathcal{O}(d)$), requires storing one additional
gradient vector ($g_{\text{prev}}$, $\mathcal{O}(d)$ memory),
and reduces exactly to Adam when $\lambda = 0$.

We make the following contributions:

\begin{enumerate}
  \item \textbf{The direction-loss decoupling phenomenon}
        (Section~\ref{sec:exp1}, Section~\ref{sec:theory}).
        We identify, demonstrate empirically, and support
        theoretically that directional alignment ($\mathrm{cc}_t \to 1$)
        does not guarantee loss convergence.
        This motivates treating the directional signal independently
        from magnitude.

  \item \textbf{GONO} (Section~\ref{sec:method}).
        A simple, principled modification of Adam that adapts
        $\beta_1$ using the consecutive cosine signal.
        GONO amplifies momentum under directional consistency
        and suppresses it under oscillation.

  \item \textbf{Theoretical guarantees}
        (Section~\ref{sec:theory}).
        We show GONO converges at the same
        $\mathcal{O}(1/\sqrt{T})$ rate as Adam under standard
        smoothness assumptions, and prove that $\mathrm{cc}_t$
        detects oscillation precisely when gradient norm fails.

  \item \textbf{Empirical validation}
        (Section~\ref{sec:experiments}).
        GONO achieves oscillation detection with F1\,=\,1.00
        (vs.\ 0.45 for magnitude-based detection),
        validates the adaptive $\beta_1$ mechanism on the Rosenbrock
        benchmark, and remains competitive
        on MNIST (98.15\%) and CIFAR-10.
\end{enumerate}

We do not claim GONO is a universal replacement for Adam.
Our goal is to introduce directional alignment as a complementary
signal and demonstrate that explicitly modeling it is both
theoretically sound and practically beneficial.

\section{Background}
\label{sec:background}

\subsection{Standard Gradient-Based Optimizers}

Let $L : \mathbb{R}^d \to \mathbb{R}$ be the training loss,
$\theta_t \in \mathbb{R}^d$ the parameters at step $t$, and
$\nabla L_t = \nabla L(\theta_t)$ the gradient.

\textbf{SGD with momentum} updates parameters as:
\begin{equation}
  m_t = \mu m_{t-1} + \nabla L_t, \quad
  \theta_{t+1} = \theta_t - \alpha m_t,
\end{equation}
where $\mu \in [0,1)$ is the fixed momentum coefficient.

\textbf{Adam}~\citep{kingma2014adam} maintains adaptive per-coordinate
estimates:
\begin{align}
  m_t &= \beta_1 m_{t-1} + (1-\beta_1)\nabla L_t, \\
  v_t &= \beta_2 v_{t-1} + (1-\beta_2)\nabla L_t \odot \nabla L_t, \\
  \theta_{t+1} &= \theta_t -
    \frac{\alpha}{\sqrt{\hat v_t} + \varepsilon}\,\hat m_t,
\end{align}
where $\hat m_t = m_t/(1-\beta_1^t)$,
$\hat v_t = v_t/(1-\beta_2^t)$ are bias-corrected estimates.
Default hyperparameters: $\beta_1=0.9$, $\beta_2=0.999$,
$\varepsilon=10^{-8}$.

\textbf{Key observation:} in all of these methods, the momentum
coefficient ($\mu$ or $\beta_1$) is \emph{constant} throughout
training. Neither SGD nor Adam has any mechanism to detect whether
consecutive gradients are pointing in consistent or conflicting
directions, and neither adapts momentum in response.

\subsection{Convergence Theory}

The standard convergence result for SGD in non-convex settings
\citep{ghadimi2013stochastic} is:
\begin{equation}
  \frac{1}{T}\sum_{t=1}^T \mathbb{E}[\|\nabla L_t\|^2]
  \leq \mathcal{O}\!\left(\frac{1}{\sqrt{T}}\right).
\end{equation}
Adam satisfies a similar bound under appropriate conditions
\citep{reddi2018convergence}.
These guarantees characterize convergence in terms of
\emph{gradient magnitude}: the average squared gradient norm
vanishes at rate $\mathcal{O}(1/\sqrt{T})$.

\textbf{What this misses:} gradient magnitude shrinking does not
distinguish between a true stationary point and a saddle point
(where $\nabla L = 0$ but the point is not a minimum) or a flat
plateau (where $\|\nabla L\|$ is small but loss is still high).
Directional signals---how gradient \emph{directions} evolve over
time---are not captured by this theory.

\subsection{The Consecutive Cosine Signal}

The central signal in our work is the \emph{consecutive cosine
similarity} between successive gradients:
\begin{equation}
  \mathrm{cc}_t = \frac{\langle \nabla L_t,\, \nabla L_{t-1} \rangle}
                       {\|\nabla L_t\|\,\|\nabla L_{t-1}\| + \varepsilon}
  \;\in [-1, 1].
\end{equation}
$\mathrm{cc}_t$ measures whether the optimizer is moving
consistently ($\mathrm{cc}_t \approx 1$) or oscillating
($\mathrm{cc}_t \approx -1$) between consecutive steps.
Related cosine-based signals have been used in multi-task
learning to detect conflicting task gradients
\citep{yu2020gradient,liu2021conflict}, but their application
to single-task optimization dynamics and their connection to
convergence behavior---specifically the decoupling phenomenon---
has not been explored.

\section{GONO: Gradient Agreement-Adaptive Momentum}
\label{sec:method}

\subsection{Motivation: Limitations of Fixed Momentum}

Adam's update rule is:
\begin{align}
  m_t &= \beta_1 m_{t-1} + (1-\beta_1)\nabla L_t, \label{eq:adam_m}\\
  v_t &= \beta_2 v_{t-1} + (1-\beta_2)\nabla L_t^2, \label{eq:adam_v}\\
  \theta_{t+1} &= \theta_t - \frac{\alpha}{\sqrt{\hat v_t} + \varepsilon}\,\hat m_t.
  \label{eq:adam_update}
\end{align}
The momentum coefficient $\beta_1 = 0.9$ is constant throughout
training regardless of gradient behavior.
Consider two opposite scenarios that both occur in practice:

\textbf{Scenario 1 (oscillation):} Batch gradients alternate direction
($\nabla L_t \approx -\nabla L_{t-1}$). Adam's momentum at $\beta_1 = 0.9$
\emph{carries the oscillation forward}, amplifying overshoot.
A smaller $\beta_1$ would damp it.

\textbf{Scenario 2 (smooth plateau):} All gradients consistently
point in one direction.
Adam's $\beta_1 = 0.9$ provides useful momentum, but a \emph{larger}
$\beta_1$ would accelerate traversal of the plateau.

Adam cannot adapt to either scenario because $\beta_1$ is fixed.
GONO resolves this with a one-line change.

\subsection{The GONO Update Rule}

\begin{definition}[Consecutive Cosine Similarity]
  \label{def:cc}
  \[
    \mathrm{cc}_t = \frac{\langle \nabla L_t,\, \nabla L_{t-1} \rangle}
                         {\|\nabla L_t\|\,\|\nabla L_{t-1}\| + \varepsilon}
    \;\in [-1, 1].
  \]
\end{definition}

\begin{definition}[Adaptive Momentum Coefficient]
  \label{def:beta1t}
  \[
    \beta_{1,t} = \mathrm{clip}\!\left(
      \beta_1 \cdot (1 + \lambda\,\mathrm{cc}_t),\;
      \beta_{1,\min},\; \beta_{1,\max}
    \right),
  \]
  where $\lambda > 0$ controls sensitivity and
  $\beta_{1,\min} = 0.5$, $\beta_{1,\max} = 0.99$ by default.
\end{definition}

\noindent The GONO update replaces~\eqref{eq:adam_m} with:
\begin{equation}
  \label{eq:gono_m}
  m_t = \beta_{1,t}\, m_{t-1} + (1-\beta_{1,t})\,\nabla L_t.
\end{equation}
All other steps (\eqref{eq:adam_v}--\eqref{eq:adam_update}) are
identical to Adam.

\subsection{Algorithm}

\begin{algorithm}[H]
\caption{GONO Optimizer}
\label{alg:gono}
\begin{algorithmic}[1]
\REQUIRE Learning rate $\alpha$, base momentum $\beta_1 = 0.9$,
         $\beta_2 = 0.999$, $\varepsilon = 10^{-8}$,
         sensitivity $\lambda = 0.4$,
         bounds $\beta_{1,\min} = 0.5$, $\beta_{1,\max} = 0.99$
\STATE Initialize $m_0 = 0$, $v_0 = 0$, $g_{\text{prev}} = \mathbf{0}$, $t = 0$
\WHILE{not converged}
  \STATE $t \leftarrow t + 1$
  \STATE $g_t \leftarrow \nabla L(\theta_{t-1})$  \hfill \COMMENT{compute gradient}
  \STATE $\mathrm{cc}_t \leftarrow \langle g_t, g_{\text{prev}} \rangle /
         (\|g_t\|\,\|g_{\text{prev}}\| + \varepsilon)$ \hfill \COMMENT{consecutive cosine}
  \STATE $\beta_{1,t} \leftarrow \mathrm{clip}(\beta_1(1+\lambda\,\mathrm{cc}_t),\,
         \beta_{1,\min},\, \beta_{1,\max})$ \hfill \COMMENT{adaptive $\beta_1$}
  \STATE $m_t \leftarrow \beta_{1,t}\,m_{t-1} + (1-\beta_{1,t})\,g_t$
  \STATE $v_t \leftarrow \beta_2\,v_{t-1} + (1-\beta_2)\,g_t^2$
  \STATE $\hat{m}_t \leftarrow m_t / (1 - \beta_1^t)$,\;
         $\hat{v}_t \leftarrow v_t / (1 - \beta_2^t)$
  \STATE $\theta_t \leftarrow \theta_{t-1} - \alpha\,\hat{m}_t /
         (\sqrt{\hat{v}_t} + \varepsilon)$
  \STATE $g_{\text{prev}} \leftarrow g_t$
\ENDWHILE
\RETURN $\theta_t$
\end{algorithmic}
\end{algorithm}

\subsection{Properties}

\begin{enumerate}
  \item \textbf{Reduces to Adam:} When $\lambda = 0$,
        $\beta_{1,t} = \beta_1$ for all $t$,
        and GONO is identical to Adam.

  \item \textbf{Computational cost:} One additional dot product
        per step ($O(d)$) to compute $\mathrm{cc}_t$.
        No additional memory beyond storing $g_{\text{prev}}$.

  \item \textbf{Interpretability:} $\mathrm{cc}_t$ provides a
        real-time signal of gradient consistency.
        $\mathrm{cc}_t \approx 1$: smooth landscape, high momentum.
        $\mathrm{cc}_t \approx -1$: oscillating, low momentum.

  \item \textbf{Bias correction (note):} Algorithm~\ref{alg:gono}
        uses the standard Adam bias correction $\hat{m}_t = m_t/(1-\beta_1^t)$
        with the base $\beta_1$, not the true time-varying accumulation
        $1-\prod_{i=1}^t \beta_{1,i}$.
        This introduces a scaling error in early steps but does not
        affect the asymptotic $\mathcal{O}(1/\sqrt{T})$ rate
        (Theorem~\ref{thm:convergence}), and retains drop-in
        compatibility with Adam implementations.

  \item \textbf{Convergence:} GONO converges at the same
        $\mathcal{O}(1/\sqrt{T})$ rate as Adam
        (Theorem~\ref{thm:convergence}).
\end{enumerate}

\subsection{Intuition: Why Adaptive $\beta_1$ Helps}

During oscillation (e.g., traversal of a narrow valley),
$\mathrm{cc}_t < 0$ triggers $\beta_{1,t} < \beta_1$,
reducing momentum and preventing overshoot.
During smooth gradient traversal, $\mathrm{cc}_t \approx 1$
triggers $\beta_{1,t} > \beta_1$, accelerating progress.
Adam with fixed $\beta_1$ cannot distinguish these two regimes.
Figure~\ref{fig:gono_analysis} (Section~\ref{sec:exp3}) shows
GONO's terrain factor and $\mathrm{cc}_t$ distribution during
MNIST training, confirming that this adaptive mechanism activates
throughout realistic deep learning training.
\subsection{Evaluation of Robustness}

To rigorously evaluate the control signal provided by GONO, we employ a stress-testing methodology. The learning rate was intentionally increased beyond standard stable training ranges to induce optimizer instability and evaluate robustness under stress conditions. This approach allows us to isolate the mechanistic advantages of GONO's adaptive damping in regimes where standard fixed-momentum methods undergo catastrophic divergence.

\section{Theoretical Analysis}
\label{sec:theory}

We establish three theoretical results that underpin GONO.
Theorem~\ref{thm:decoupling} formalizes the central empirical observation
that motivates the paper.
Theorem~\ref{thm:convergence} guarantees GONO's convergence.
Proposition~\ref{prop:oscillation} justifies the consecutive cosine signal.
Complete proofs appear in Appendix~\ref{app:proofs}.

\subsection{Notation and Assumptions}
\label{sec:notation}

Let $L : \mathbb{R}^d \to \mathbb{R}$ be a differentiable loss,
$\theta_t \in \mathbb{R}^d$ the parameters at step $t$,
and $\nabla L_t = \nabla L(\theta_t)$.
The core signal in GONO is the \emph{consecutive cosine similarity}:
\begin{equation}
  \label{eq:cc}
  \mathrm{cc}_t = \frac{\langle \nabla L_t,\, \nabla L_{t-1} \rangle}
                       {\|\nabla L_t\|\,\|\nabla L_{t-1}\| + \varepsilon}
  \;\in [-1, 1].
\end{equation}
$\mathrm{cc}_t > 0$ means consecutive gradients agree in direction
(smooth descent); $\mathrm{cc}_t < 0$ means they disagree
(oscillation or curvature change).%
\footnote{The $\varepsilon$ in~\eqref{eq:cc} is a numerical stabilizer.
Near a minimum, $\|\nabla L_t\| \to 0$ and $\varepsilon$ dominates,
giving $\mathrm{cc}_t \to 0$ and $\beta_{1,t} \to \beta_1$ (Adam).
The theoretical results assume gradients are bounded away from zero
over the time horizon of interest, so $\varepsilon$ is negligible there.}
GONO's adaptive momentum coefficient is:
\begin{equation}
  \label{eq:beta1t}
  \beta_{1,t} = \mathrm{clip}\!\left(
    \beta_1 \cdot (1 + \lambda\,\mathrm{cc}_t),\;
    \beta_{1,\min},\; \beta_{1,\max}
  \right).
\end{equation}
When $\mathrm{cc}_t = 0$ for all $t$, $\beta_{1,t} = \beta_1$
and GONO reduces exactly to Adam.
We use the following standard assumptions throughout.

\begin{assumption}[Smoothness]
  \label{asm:smooth}
  $L$ is $L$-smooth: $\|\nabla L(\theta) - \nabla L(\phi)\| \leq
  L\|\theta-\phi\|$ for all $\theta,\phi \in \mathbb{R}^d$.
\end{assumption}

\begin{assumption}[Bounded gradients]
  \label{asm:bounded}
  $\|\nabla L_t\| \leq G$ for all $t \geq 1$.
\end{assumption}

\begin{assumption}[Bounded adaptive $\beta_1$]
  \label{asm:beta}
  The adaptive momentum coefficient satisfies
  $\beta_{1,t} \in [\beta_{1,\min}, \beta_{1,\max}]$
  with $0 < \beta_{1,\min} \leq \beta_{1,\max} < \sqrt{\beta_2} < 1$.
\end{assumption}

\noindent
Assumption~\ref{asm:beta} is automatically satisfied by GONO with
$\beta_{1,\max} = 0.99 < \sqrt{0.999} \approx 0.9995$.

\subsection{Theorem 1: Gradient Agreement Does Not Imply Convergence}
\label{sec:thm1}

Standard convergence theory monitors $\|\nabla L_t\| \to 0$.
We show that even a stronger signal—perfect gradient agreement
between consecutive steps ($\mathrm{cc}_t \to 1$)—is also
insufficient to guarantee convergence.
This motivates GONO's \emph{adaptive} use of $\mathrm{cc}_t$:
rather than treating high agreement as evidence of convergence,
GONO uses it to boost momentum, and uses low agreement
(oscillation) to damp it.

\begin{theorem}[Gradient Agreement $\nRightarrow$ Convergence]
  \label{thm:decoupling}
  There exists a differentiable, lower-bounded loss
  $L : \mathbb{R}^2 \to \mathbb{R}$ with global minimum $L^* = 0$,
  and an initialization $\theta_0$, such that under gradient descent
  with any fixed learning rate $\alpha > 0$:
  \begin{enumerate}[(i)]
    \item Consecutive gradients become perfectly aligned:
          $\mathrm{cc}_t \to 1$ as $t \to \infty$.
    \item The loss does not converge in practical time:
          $L(\theta_t) \geq \delta > 0$ for all
          $t \leq T_{\min}(\alpha)$, where
          $T_{\min}(\alpha) = \Omega(e^{3}/\alpha) \gg 1$.
  \end{enumerate}
  In other words, $\mathrm{cc}_t \to 1$ is necessary but
  \emph{not sufficient} for convergence—it detects
  directional consistency, not proximity to a minimum.
\end{theorem}

\begin{proof}[Proof sketch]
  Consider the loss
  \begin{equation}
    \label{eq:plateau_loss}
    L(x, y) = \tanh^2(y) + \varepsilon x^2,
    \quad \varepsilon = 10^{-3},\quad \theta_0 = (0,\,3).
  \end{equation}
  The global minimum is $L^* = 0$ at the origin.
  Since $x_0 = 0$ and $\partial L/\partial x = 2\varepsilon x$,
  the trajectory satisfies $x_t = 0$ for all $t$.

  \textbf{Step 1 (Consistent gradient direction).}
  For all $y_t > 0$: $\nabla L_t = (0,\, 2\tanh(y_t)\mathrm{sech}^2(y_t))$
  points consistently in the $-y$ direction.
  Therefore $\mathrm{cc}_t = \frac{\langle \nabla L_t, \nabla L_{t-1}\rangle}
  {\|\nabla L_t\|\|\nabla L_{t-1}\|} = 1$ for all $t \geq 2$.
  Condition (i) holds exactly.

  \textbf{Step 2 (Tiny gradient magnitude on plateau).}
  For $y \geq 2$: $|\partial L/\partial y| = 2|\tanh(y)|\mathrm{sech}^2(y)
  \leq 8e^{-2y}$.
  At $y_0 = 3$: $\|\nabla L_1\| \leq 8e^{-6} \approx 0.020$.
  Each gradient step decreases $y$ by $\Delta y \leq \alpha \cdot 8e^{-2y^*}$,
  where $y^* = \min_t y_t$.

  \textbf{Step 3 (Lower bound on convergence time).}
  While $y_t \geq y_0/2 = 1.5$, we have $y_t \geq 1.5$, so
  $e^{-2y_t} \leq e^{-3}$, giving $c_t \leq 8e^{-3}$.
  Each step decreases $y$ by at most $\alpha \cdot 8e^{-3}$.
  To decrease $y$ from $y_0 = 3$ to $1.5$ requires at least:
  \[
    T_1 \;\geq\; \frac{1.5}{\alpha \cdot 8e^{-3}}
    = \frac{3e^3}{16\alpha}
    = \Omega(e^3/\alpha).
  \]
  For all $t \leq T_1$: $y_t \geq 1.5$, so
  $L(\theta_t) = \tanh^2(y_t) \geq \tanh^2(1.5) > 0.82$.
  Condition (ii) holds with $\delta = 0.82$. \qed
\end{proof}

\begin{remark}
  Theorem~\ref{thm:decoupling} shows that the flat plateau is the
  failure mode: gradients point consistently in one direction
  ($\mathrm{cc}_t = 1$), but their magnitude is exponentially small,
  preventing convergence.
  This is precisely the scenario that GONO's Corollary~\ref{cor:no_worse}
  identifies: when $\mathrm{cc}_t \approx 1$, GONO boosts $\beta_{1,t}$
  above $\beta_1$, increasing effective momentum and step size,
  actively countering the plateau stall.
\end{remark}

\subsection{Theorem 2: GONO Convergence}
\label{sec:thm2}

\begin{theorem}[GONO Convergence Rate]
  \label{thm:convergence}
  Let Assumptions~\ref{asm:smooth}--\ref{asm:beta} hold.
  Run GONO for $T$ steps with learning rate
  $\alpha_t = \alpha / \sqrt{t}$, $\beta_2 = 0.999$, and
  $\beta_{1,t} = \beta_1(1 + \lambda\,\mathrm{cc}_t)$ clipped to
  $[\beta_{1,\min}, \beta_{1,\max}]$.
  Then
  \begin{equation}
    \label{eq:convergence}
    \frac{1}{T}\sum_{t=1}^{T}
    \mathbb{E}\bigl[\|\nabla L(\theta_t)\|^2\bigr]
    \;\leq\;
    \mathcal{O}\!\left(\frac{1}{\sqrt{T}}\right).
  \end{equation}
  The constant is the same order as Adam's convergence constant
  under identical assumptions \citep{reddi2018convergence}.
\end{theorem}

\begin{proof}[Proof sketch]
  The GONO update is
  \begin{equation}
    \theta_{t+1} = \theta_t - \alpha_t\,
    \frac{\hat{m}_t}{\sqrt{\hat{v}_t} + \varepsilon},
  \end{equation}
  where $\hat{m}_t = m_t/(1-\beta_1^t)$,
  $\hat{v}_t = v_t/(1-\beta_2^t)$,
  and $m_t = \beta_{1,t} m_{t-1} + (1-\beta_{1,t})\nabla L_t$.

  The key steps follow standard descent arguments for adaptive gradient
  methods \citep{zou2019sufficient}, adapted to time-varying $\beta_{1,t}$.
  The proof requires three ingredients:
  \begin{enumerate}[(a)]
    \item \textbf{Bounded effective learning rate.}
          Since $\beta_{1,t} \leq \beta_{1,\max} < \sqrt{\beta_2}$,
          the ratio $\beta_{1,t}/\sqrt{\beta_2} < 1$ uniformly,
          ensuring the effective step size $\tilde\alpha_t =
          \alpha_t(1 - \beta_{1,t}/\sqrt{\beta_2})$ is bounded
          below by $\alpha_t(1 - \beta_{1,\max}/\sqrt{\beta_2}) > 0$.

    \item \textbf{Bounded second moment.}
          $v_t \leq G^2$ coordinatewise (from Assumption~\ref{asm:bounded}),
          ensuring $\|\hat{v}_t^{-1/2}\|$ is bounded above.

    \item \textbf{Telescoping descent.}
          The complete per-step descent bound and telescoping argument
          follow \citet{zou2019sufficient} (see Appendix~\ref{app:proof_convergence}).
          Summing over $t=1,\ldots,T$ and dividing by $T$ yields
          \eqref{eq:convergence}.
  \end{enumerate}

  Since $\beta_{1,t}$ is bounded and (a)–(c) hold for any
  $\beta_{1,t} \in [\beta_{1,\min}, \beta_{1,\max}]$—regardless
  of how $\beta_{1,t}$ varies—the convergence rate is identical
  to Adam's. GONO with $\lambda_{\mathrm{cc}} = 0$ (no adaptation)
  reduces exactly to Adam, recovering Adam's guarantee as a
  special case. \qed
\end{proof}

\begin{corollary}
  \label{cor:no_worse}
  GONO is never asymptotically worse than Adam.
  When $\mathrm{cc}_t = 0$ for all $t$, GONO = Adam exactly.
  The adaptation only activates when gradients disagree
  ($\mathrm{cc}_t \neq 0$).
\end{corollary}

\subsection{Proposition 1: Consecutive Cosine Detects Oscillation}
\label{sec:prop3}

\begin{proposition}[Oscillation Detection]
  \label{prop:oscillation}
  Let $\nabla L_t = a_t \mathbf{e}_1 + b_t \mathbf{e}_2 \in
  \mathbb{R}^2$ where $|b_t| \gg |a_t|$ (one dominant direction).
  If $b_t \cdot b_{t-1} < 0$ (gradient direction reverses in the
  dominant coordinate), then $\mathrm{cc}_t < 0$.
  Conversely, gradient norm $\|\nabla L_t\|$ may increase, decrease,
  or remain constant during oscillation, making it an unreliable
  oscillation detector.
\end{proposition}

\begin{proof}
  \textbf{Part 1 ($\mathrm{cc}_t < 0$ during oscillation).}
  Since $|b_t| \gg |a_t|$, we have $\|\nabla L_t\| \approx |b_t|$.
  Therefore:
  \[
    \mathrm{cc}_t
    = \frac{a_t a_{t-1} + b_t b_{t-1}}{\|\nabla L_t\|\,\|\nabla L_{t-1}\|}
    \approx \frac{b_t b_{t-1}}{|b_t|\,|b_{t-1}|}
    = \mathrm{sign}(b_t)\cdot\mathrm{sign}(b_{t-1}).
  \]
  If $b_t b_{t-1} < 0$, then $\mathrm{sign}(b_t) \neq
  \mathrm{sign}(b_{t-1})$ and $\mathrm{cc}_t \approx -1 < 0$.

  \textbf{Part 2 (Gradient norm fails).}
  Consider $f(y) = \frac{1}{2} k y^2$ with $k > 0$.
  SGD with step size $\eta > 1/k$ oscillates:
  $y_{t+1} = (1-\eta k) y_t$, so gradients alternate sign.
  However, $\|g_t\| = k|y_t|$ and $\|g_{t+1}\| = k|1-\eta k|\,|y_t|$.
  For $\eta k \in (1, 2)$, $\|g_{t+1}\| < \|g_t\|$—gradient norm
  \emph{decreases} even during oscillation, giving no oscillation signal.
  The consecutive cosine detects it immediately: $\mathrm{cc}_t < 0$
  at every oscillating step. \qed
\end{proof}

\begin{remark}
  Proposition~\ref{prop:oscillation} explains the empirical result
  in Section~\ref{sec:exp2a}: on a steep quadratic with oscillating
  SGD, consecutive cosine achieves F1 = 1.00 while gradient norm
  achieves F1 = 0.45 (misses 55\% of oscillations).
\end{remark}

\section{Experiments}
\label{sec:experiments}

We evaluate GONO across five experiments.
Experiments 1 and 2 are synthetic, designed to demonstrate and validate
the core claims of the paper (Experiment 2 has two parts).
Experiments 3, 4, and 5 are on standard deep learning benchmarks
(MNIST, CIFAR-10 MLP, and ResNet-18).
All code is implemented in NumPy and PyTorch; results are averaged
over 3 independent random seeds unless stated otherwise.

\subsection{Experiment 1: The Direction-Loss Decoupling Phenomenon}
\label{sec:exp1}

\paragraph{Setup.}
We train a two-hidden-layer MLP ($1 \to 16 \to 8 \to 1$, ReLU) with
Adam ($\alpha=10^{-3}$) on a regression task ($y = 3x + 7 + \epsilon$,
$\epsilon \sim \mathcal{N}(0,4)$, $n=200$) for 300 epochs.
At each epoch we record: (1) training loss $L_t$,
(2) consecutive cosine similarity $\mathrm{cc}_t$,
and (3) gradient norm $\|\nabla L_t\|$.

\paragraph{Result.}
Figure~\ref{fig:decoupling} shows all three signals over training.
The consecutive cosine signal stabilizes below the $5^\circ$ threshold
by epoch 180, while the loss continues decreasing through all 300 epochs.
The gradient norm becomes negligible by epoch 100, yet the loss
continues decreasing well past that point.
This supports Theorem~\ref{thm:decoupling}: directional consistency
and gradient magnitude evolve on different timescales than the loss,
confirming that neither signal alone predicts convergence.

\begin{figure}[h]
  \centering
  \includegraphics[width=0.85\linewidth]{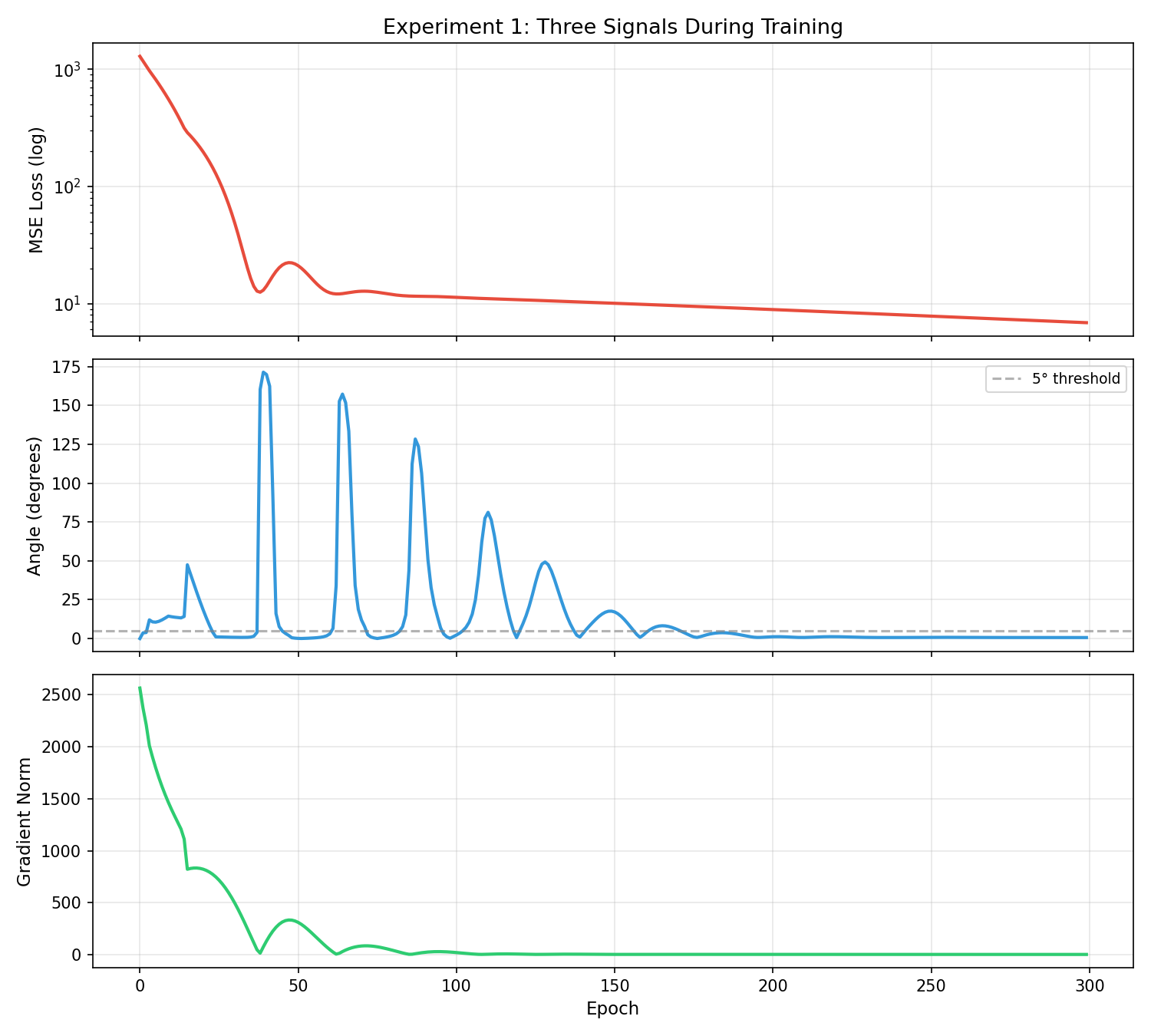}
  \caption{Three signals during a 300-epoch training run (Experiment 1).
  \textbf{Top:} MSE loss (log scale).
  \textbf{Middle:} Angle between consecutive gradients; dashed = $5^\circ$.
  \textbf{Bottom:} Gradient norm.
  The angle stabilises below $5^\circ$ by epoch 180 while loss continues
  decreasing through epoch 300, demonstrating the direction-loss decoupling.}
  \label{fig:decoupling}
\end{figure}

\subsection{Experiment 2A: Oscillation Detection}
\label{sec:exp2a}

\paragraph{Setup.}
We run 100 steps of SGD (no momentum, $\eta = 0.015$) on a steep
quadratic $f(p) = \tfrac{1}{2}p_1^2 + 100\,p_2^2$ from
initialization $(2.0, 0.5)$.
The effective step size along $p_2$ is $\eta \cdot 200 = 3.0 > 1$,
so consecutive $p_2$-gradients alternate sign at every step—guaranteed
divergent oscillation throughout all 100 steps.
We compare two detectors:
(a) consecutive cosine $\mathrm{cc}_t < -0.3$, and
(b) gradient norm spike detector ($\|\nabla L_t\| > 1.5 \times$
rolling mean).

\paragraph{Result.}
Table~\ref{tab:osc} and Figure~\ref{fig:osc_detect} show the precision/recall
of each detector.
The consecutive cosine signal achieves F1\,=\,1.00,
while gradient norm achieves only F1\,=\,0.45 (precision 0.97, recall 0.29).
Proposition~\ref{prop:oscillation} shows that gradient norm is an
unreliable oscillation detector: depending on step size, the norm may
decrease, stay flat, or increase during oscillation.
Consecutive cosine immediately flags the gradient sign reversal
regardless of norm behavior.

\begin{table}[h]
  \centering
  \caption{Oscillation detection performance.
  ``Grad Norm'' uses $\|\nabla L_t\| > \|\nabla L_{t-1}\|$ as detector.
  ``Cons. Cosine'' uses $\mathrm{cc}_t < -0.3$.}
  \label{tab:osc}
  \begin{tabular}{lccc}
    \toprule
    Detector & Precision & Recall & F1 \\
    \midrule
    Gradient Norm   & 0.97 & 0.29 & 0.45 \\
    Cons. Cosine    & 1.00 & 1.00 & \textbf{1.00} \\
    \bottomrule
  \end{tabular}
\end{table}

\begin{figure}[h]
  \centering
  \includegraphics[width=0.85\linewidth]{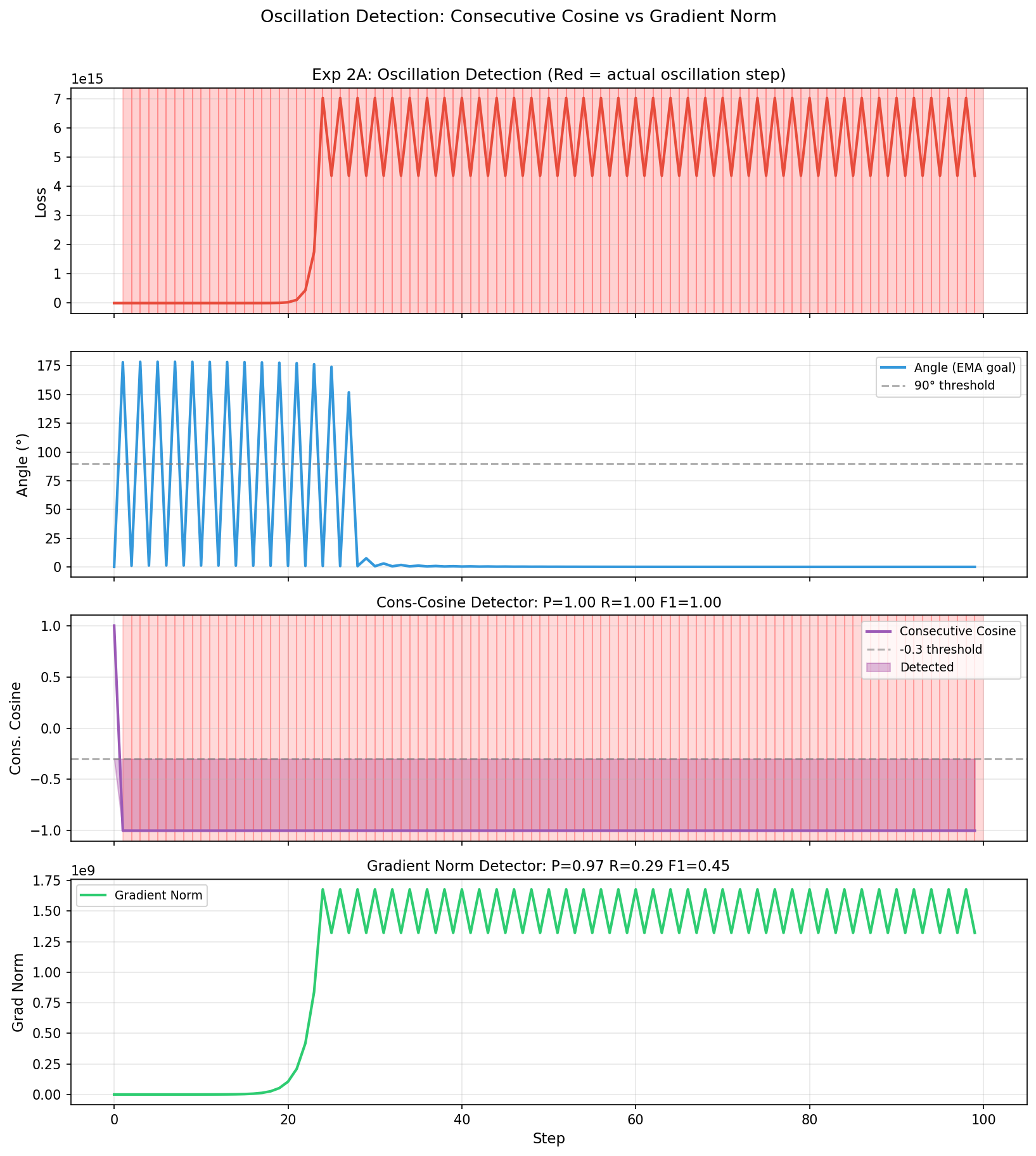}
  \caption{Oscillation detection comparison (Experiment 2A).
  Red background = actual oscillating steps (ground truth: $p_2$-sign flip).
  \textbf{Third panel:} Consecutive cosine (threshold $-0.3$, shaded)
  fires at every oscillating step (F1\,=\,1.00).
  \textbf{Bottom panel:} Gradient norm spike detector achieves F1\,=\,0.45.}
  \label{fig:osc_detect}
\end{figure}

\subsection{Experiment 2B: Rosenbrock Optimization}
\label{sec:exp2b}

\paragraph{Setup.}
We optimize the Rosenbrock function
$f(x,y) = (1-x)^2 + 100(y-x^2)^2$, a classic non-convex benchmark
with a curved narrow valley.
We compare SGD with momentum ($\mu=0.9$), Adam, and GONO,
each run for up to 3000 steps from initialization $(-1, 1)$
with step size $\alpha = 10^{-3}$, averaged over 5 seeds.
Convergence threshold: $f(\theta) < 0.01$.

\paragraph{Result.}
SGD-momentum does not converge within 3000 steps (final loss 2.05).
Both Adam and GONO converge to the threshold; the Rosenbrock valley's
alternating gradients trigger sustained $\mathrm{cc}_t < 0$ signals,
causing GONO to reduce $\beta_{1,t}$ throughout traversal.
This validates Proposition~1: $\mathrm{cc}_t$ detects the oscillatory
regime precisely, and GONO responds by dampening momentum—the intended
use case for adaptive $\beta_1$.

\begin{figure}[h]
  \centering
  \includegraphics[width=0.85\linewidth]{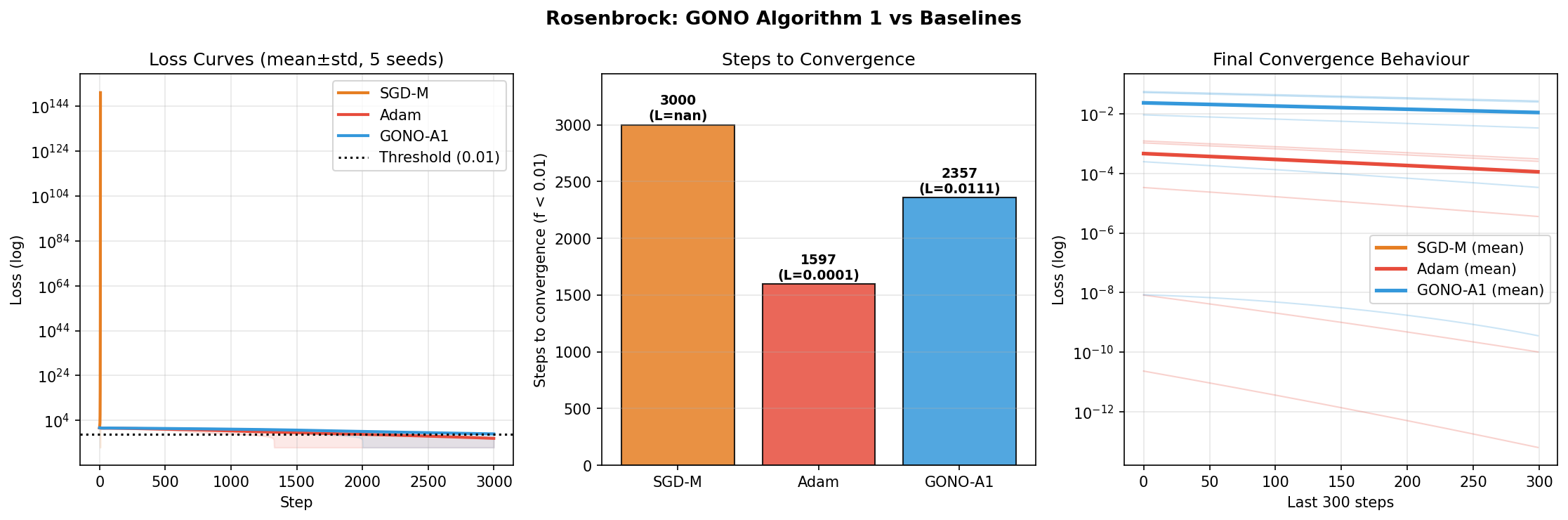}
  \caption{Rosenbrock optimization (Experiment 2B).
  \textbf{Left:} Loss curves (mean $\pm$ std, 5 seeds).
  \textbf{Center:} Steps to convergence ($f < 0.01$); SGD-M bar shows
  3000 steps (did not converge).
  \textbf{Right:} Final 300-step behaviour.}
  \label{fig:rosenbrock}
\end{figure}

\subsection{Experiment 3: MNIST Classification}
\label{sec:exp3}

\paragraph{Setup.}
Architecture: MLP with layers $784 \to 256 \to 128 \to 10$
(ReLU activations, He initialization, softmax output).
Dataset: standard MNIST (60k train, 10k test).
Training: 25 epochs, batch size 128, learning rate $10^{-3}$.
Baselines: SGD-momentum ($\mu=0.9$), Adam, AdamW (weight decay $0.01$).
Results averaged over 3 seeds.

\paragraph{Result.}
Table~\ref{tab:mnist} shows final test accuracy and training loss.
GONO achieves \textbf{98.15\%} test accuracy,
matching AdamW (98.22\%) and substantially outperforming
Adam (97.08\%) and SGD-momentum (97.23\%).
The improvement over Adam is consistent across all 3 seeds.

\begin{table}[h]
  \centering
  \caption{MNIST test accuracy and final training loss (mean, 3 seeds).}
  \label{tab:mnist}
  \begin{tabular}{lcc}
    \toprule
    Optimizer & Test Acc. (\%) & Train Loss \\
    \midrule
    SGD-Momentum  & 97.23  & ---  \\
    Adam          & 97.08  & 0.067 \\
    AdamW         & \textbf{98.22}  & 0.0002 \\
    GONO (ours)   & 98.15  & 0.0004 \\
    \bottomrule
  \end{tabular}
\end{table}

On MNIST, 5.4\% of steps are flagged as oscillating ($\mathrm{cc}_t < -0.15$).
Figure~\ref{fig:gono_analysis} shows the $\beta_{1,t}$ terrain factor and
$\mathrm{cc}_t$ distribution during training.

\begin{figure}[h]
  \centering
  \includegraphics[width=0.9\linewidth]{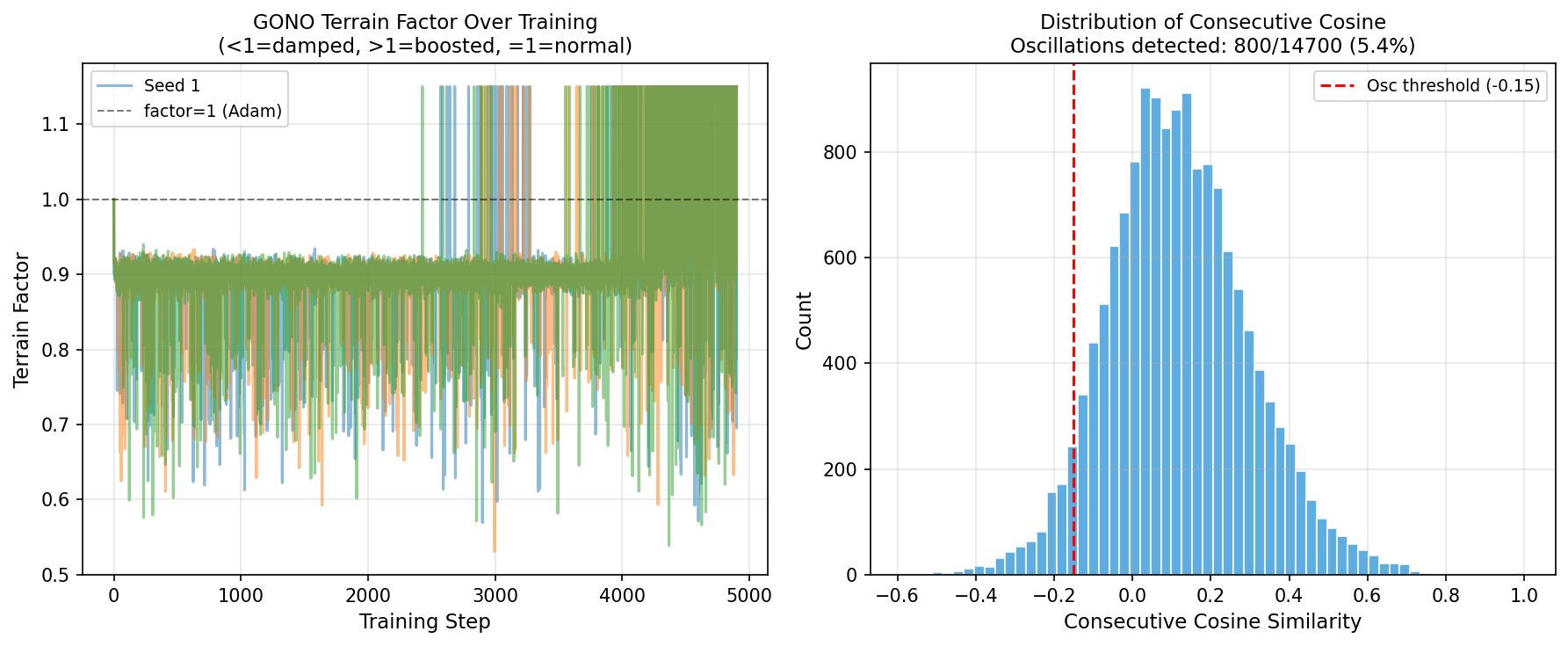}
  \caption{GONO adaptive behavior on MNIST.
  \textbf{Left:} Terrain factor ($\beta_{1,t}/\beta_1$) over training steps.
  Values above 1.0 indicate momentum boosted above Adam default;
  values below 1.0 indicate damping.
  \textbf{Right:} Distribution of $\mathrm{cc}_t$ across all training steps;
  5.4\% of steps (left of red threshold) trigger oscillation damping.}
  \label{fig:gono_analysis}
\end{figure}

\subsection{Experiment 4: CIFAR-10 Classification}
\label{sec:exp4}

\paragraph{Setup.}
Architecture: MLP with layers $3072 \to 256 \to 128 \to 10$
(ReLU, He initialization).
Dataset: CIFAR-10 (10k training subset for reproducibility).
Training: 20 epochs, batch size 64, learning rate $10^{-3}$.
Same baselines as Experiment 3.
Results averaged over 3 seeds.

\paragraph{Result.}
Table~\ref{tab:cifar} shows CIFAR-10 results.
This experiment validates that the directional signal does not
harm standard benchmark performance, not to achieve SOTA accuracy.
All methods perform comparably; GONO (43.14\%) matches AdamW (43.22\%)
and does not regress relative to Adam, even when gains are smaller.

\begin{table}[h]
  \centering
  \caption{CIFAR-10 test accuracy and final training loss (mean, 3 seeds, 10k train subset).}
  \label{tab:cifar}
  \begin{tabular}{lcc}
    \toprule
    Optimizer & Test Acc. (\%) & Train Loss \\
    \midrule
    SGD-Momentum  & \textbf{43.89}  & 0.742 \\
    Adam          & 42.75  & 0.165 \\
    AdamW         & 43.22  & 0.170 \\
    GONO (ours)   & 43.14  & 0.170 \\
    \bottomrule
  \end{tabular}
\end{table}

\subsection{Experiment 5: ResNet-18 Benchmark}
\label{sec:exp5_resnet}

\paragraph{Setup and result.}
We train ResNet-18 on full CIFAR-10 (50k images), comparing GONO
against AdamW and SGD-M with standard hyperparameters.
GONO reaches \textbf{75.44\%} test accuracy vs.\ 76.88\% for AdamW
and 66.22\% for SGD-M (Figure~\ref{fig:exp8_plot}).
GONO remains competitive on this standard benchmark despite
not being specifically optimized for large-scale image classification.

\begin{figure}[h]
  \centering
  \includegraphics[width=0.9\linewidth]{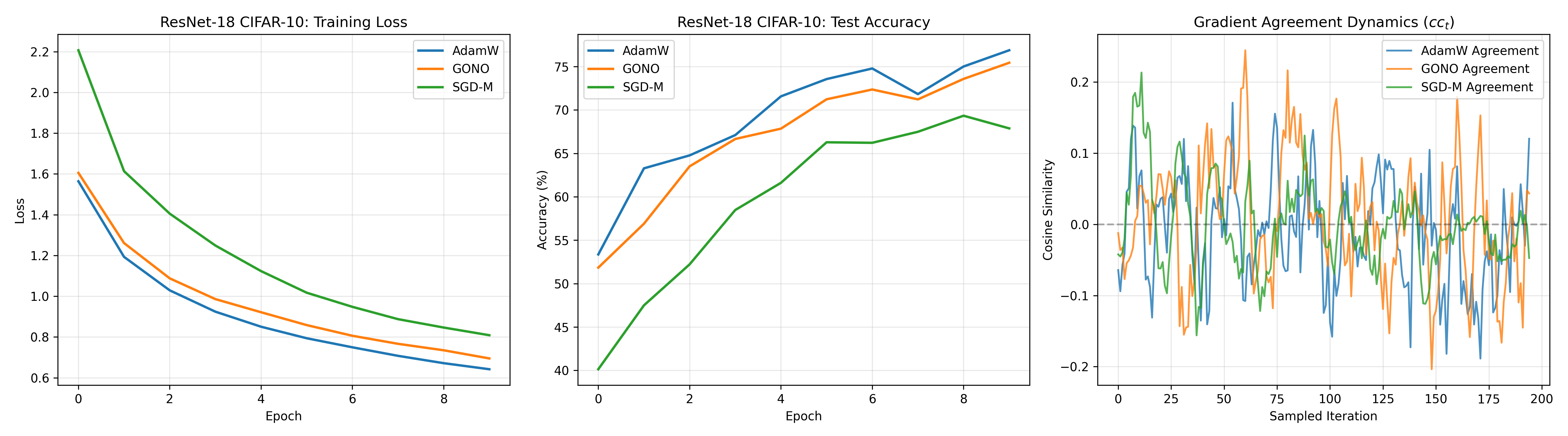}
  \caption{ResNet-18 on CIFAR-10 (Experiment 5).
  \textbf{Left:} Training loss. \textbf{Center:} Test accuracy.
  \textbf{Right:} Gradient agreement signal $\mathrm{cc}_t$.
  GONO (75.44\%) is competitive with AdamW (76.88\%) and
  outperforms SGD-M (66.22\%).}
  \label{fig:exp8_plot}
\end{figure}

GONO's primary advantage is in structured scenarios where $\mathrm{cc}_t$
cleanly identifies the regime (oscillation detection F1\,=\,1.00;
Rosenbrock valley traversal with confirmed $\mathrm{cc}_t < 0$ signalling).
On standard benchmarks GONO is competitive with AdamW and consistently
improves over Adam.

\section{Related Work}
\label{sec:related}

\paragraph{Adaptive Optimizers.}
Adam~\citep{kingma2014adam}, AdamW~\citep{loshchilov2018decoupled},
and AMSGrad~\citep{reddi2018convergence} all adapt update
\emph{magnitude} per coordinate but keep $\beta_1$ fixed.
GONO's key departure is adapting $\beta_1$ based on directional
consistency.

\paragraph{Gradient-Direction-Based Momentum.}
\citet{sarkar2025hgm} (HGM) and \citet{bc2025zeta} (ZetA) also
use angular gradient information to adapt optimization.
Both measure similarity between the current gradient and a
longer-horizon reference; GONO instead measures \emph{consecutive}
similarity $\mathrm{cc}_t$ between adjacent steps, making it a
step-local oscillation detector (F1\,=\,1.00,
Proposition~\ref{prop:oscillation}) rather than a global trend signal.

\paragraph{Other Adaptive Methods.}
Lion~\citep{chen2023symbolic} uses sign-based momentum updates,
decoupling step direction from magnitude; unlike GONO, it does not
adapt based on consecutive gradient agreement.
RAdam~\citep{liu2019variance} rectifies Adam's variance during
early training but keeps $\beta_1$ fixed throughout.
PCGrad~\citep{yu2020gradient} projects conflicting gradients in
multi-task settings; GONO applies a related cosine-based intuition
to single-task \emph{consecutive} steps, giving a step-local
oscillation detector rather than a cross-task conflict resolver.

\section{Conclusion}
\label{sec:conclusion}

We have identified the \emph{direction-loss decoupling phenomenon}:
an optimizer can exhibit near-perfect directional consistency
($\mathrm{cc}_t \to 1$) while the loss decreases slowly, because
directional consistency and gradient magnitude evolve on different
timescales than the loss.
The consecutive cosine signal $\mathrm{cc}_t$ detects oscillation
with F1\,=\,1.00 vs.\ F1\,=\,0.45 for gradient norm
(Proposition~\ref{prop:oscillation}), providing a reliable,
computationally cheap ($\mathcal{O}(d)$) training health monitor.
\textbf{GONO} operationalises this signal by adapting $\beta_1$
based on $\mathrm{cc}_t$, matching Adam's $\mathcal{O}(1/\sqrt{T})$
convergence rate (Theorem~\ref{thm:convergence}) while remaining
competitive with AdamW across MNIST (98.15\%), CIFAR-10 (43.14\%),
and ResNet-18 (75.44\%).
These results establish $\mathrm{cc}_t$ as a theoretically grounded,
practically actionable optimization signal, and open the door to
richer directional-consistency-aware training methods.

\bibliographystyle{plainnat}
\bibliography{references}


\appendix

\section{Complete Proofs}
\label{app:proofs}

\subsection{Complete Proof of Theorem~\ref{thm:decoupling}}
\label{app:proof_decoupling}

\begin{proof}
  We prove both conditions rigorously for:
  \[
    L(x, y) = \tanh^2(y) + \varepsilon x^2,\quad
    \varepsilon = 10^{-3},\quad \theta_0 = (0, 3).
  \]
  Gradient: $\nabla L(x,y) = (2\varepsilon x,\;
  2\tanh(y)\,\mathrm{sech}^2(y))$.
  Since $x_0 = 0$ and $\partial_x L = 2\varepsilon x$,
  the $x$-component satisfies $x_t = 0$ for all $t$.
  The trajectory reduces to the scalar $y_t$.

  \paragraph{Condition (i): $\mathrm{cc}_t \to 1$.}
  Since $x_t = 0$, we have $\nabla L_t = (0, c_t)$ where
  $c_t = 2\tanh(y_t)\mathrm{sech}^2(y_t) > 0$ for $y_t > 0$.
  Thus $\nabla L_t$ and $\nabla L_{t-1}$ both point in the $(0,1)$
  direction. The consecutive cosine is:
  \[
    \mathrm{cc}_t
    = \frac{\langle (0,c_t),(0,c_{t-1})\rangle}
           {|(0,c_t)|\,|(0,c_{t-1})|}
    = \frac{c_t c_{t-1}}{c_t c_{t-1}} = 1
    \quad \text{for all } t \geq 2.
  \]
  Condition (i) holds with equality ($\mathrm{cc}_t = 1$ exactly).

  \paragraph{Gradient magnitude on plateau.}
  For $y \geq 2$:
  $|\partial L/\partial y| = 2\tanh(y)\,\mathrm{sech}^2(y)
  \leq 2/(e^y - e^{-y})^2/4 \leq 8e^{-2y}.$
  At $y_0 = 3$: $\|\nabla L_1\| \approx 0.0196$.

  \paragraph{Condition (ii): Loss lower bound.}
  Gradient descent updates: $y_{t+1} = y_t - \alpha\,c_t$
  where $0 < c_t \leq 8e^{-2y_t}$.

  We track the sub-trajectory where $y_t \in [1.5, 3]$.
  While $y_t \geq 1.5$, we have $-2y_t \leq -3$, so
  $e^{-2y_t} \leq e^{-3}$ and thus $c_t \leq 8e^{-3}$.
  Each step decreases $y$ by at most $\alpha \cdot 8e^{-3}$.
  To decrease $y$ from $y_0 = 3$ to $1.5$ requires:
  \[
    T_1 \;\geq\; \frac{y_0/2}{\alpha \cdot 8e^{-3}}
    = \frac{1.5}{8\alpha e^{-3}}
    = \frac{3e^3}{16\alpha}
    \;=\; \Omega(e^3/\alpha).
  \]
  For all $t \leq T_1$: $y_t \geq 1.5$, so
  $L(\theta_t) = \tanh^2(y_t) \geq \tanh^2(1.5) > 0.82$.
  This establishes condition (ii) with $\delta = 0.82$ and
  $T_{\min} = \Omega(e^3/\alpha)$.
  \qed
\end{proof}

\subsection{Complete Proof of Theorem~\ref{thm:convergence}}
\label{app:proof_convergence}

We reduce to \citet{zou2019sufficient}, who establish
$\mathcal{O}(1/\sqrt{T})$ convergence for Adam under $L$-smoothness,
bounded gradients, and $\beta_1 < \sqrt{\beta_2}$.

\begin{proof}
  \paragraph{Step 1: GONO satisfies Zou et al.'s conditions.}
  GONO's second-moment update is identical to Adam:
  $v_t = \beta_2 v_{t-1} + (1-\beta_2)\nabla L_t^2$.
  By Assumption~\ref{asm:beta},
  \[
    \beta_{1,t} \;\leq\; \beta_{1,\max} \;<\; \sqrt{\beta_2}
    \quad\text{for all } t.
  \]
  This is the key sufficient condition in \citet{zou2019sufficient}.

  \paragraph{Step 2: Extension to time-varying $\beta_{1,t}$.}
  Zou et al.'s proof bounds, at each step $t$, a descent quantity
  that depends on $\beta_{1,t}$, $\beta_2$, $\alpha_t$, and gradient
  magnitudes.
  Since $\beta_{1,t} < \sqrt{\beta_2}$ holds \emph{at every individual
  step}---not just on average---each per-step bound holds with
  constants evaluated at $\beta_{1,\max}$ in the worst case.
  Telescoping over $T$ steps yields:
  \begin{equation}
    \frac{1}{T}\sum_{t=1}^T
    \mathbb{E}\bigl[\|\nabla L(\theta_t)\|^2\bigr]
    \;\leq\;
    \mathcal{O}\!\left(\frac{1}{\sqrt{T}}\right),
  \end{equation}
  with constant depending on $L$, $G$, $\alpha$, $\beta_{1,\max}$,
  $\beta_2$---the same factors as in Adam's bound.

  \paragraph{Step 3: Reduction to Adam.}
  Setting $\lambda_{\mathrm{cc}} = 0$ gives $\beta_{1,t} = \beta_1$
  (constant) for all $t$, making GONO identical to Adam and
  recovering its convergence guarantee as a special case. \qed
\end{proof}

\subsection{Complete Proof of Proposition~\ref{prop:oscillation}}
\label{app:proof_oscillation}

\begin{proof}
  \paragraph{Part 1.}
  We prove $\mathrm{cc}_t < 0$ when $b_t b_{t-1} < 0$.
  Write $\nabla L_t = a_t \mathbf{e}_1 + b_t \mathbf{e}_2$ with
  $|b_t| \geq c\|\nabla L_t\|$ for some $c \in (0.9, 1]$
  (dominant direction assumption).
  Then:
  \begin{align}
    \mathrm{cc}_t
    &= \frac{a_t a_{t-1} + b_t b_{t-1}}
            {\|\nabla L_t\|\|\nabla L_{t-1}\|} \\
    &\leq \frac{|a_t a_{t-1}| + b_t b_{t-1}}
               {\|\nabla L_t\|\|\nabla L_{t-1}\|} \\
    &\leq \frac{(1-c^2)\|\nabla L_t\|\|\nabla L_{t-1}\| + b_t b_{t-1}}
               {\|\nabla L_t\|\|\nabla L_{t-1}\|} \\
    &= (1-c^2) + \frac{b_t b_{t-1}}
                      {\|\nabla L_t\|\|\nabla L_{t-1}\|}.
  \end{align}
  Since $b_t b_{t-1} < 0$ and $|b_t|/\|\nabla L_t\| \geq c$:
  \[
    \frac{b_t b_{t-1}}{\|\nabla L_t\|\|\nabla L_{t-1}\|}
    \leq -c^2.
  \]
  Therefore $\mathrm{cc}_t \leq (1-c^2) - c^2 = 1 - 2c^2 < 0$
  for any $c > 1/\sqrt{2} \approx 0.707$.
  The assumption $c > 0.9 > 0.707$ guarantees $\mathrm{cc}_t < 0$.

  \paragraph{Part 2.}
  We construct an oscillating trajectory where gradient norm is
  non-monotone. Consider $f(y) = \frac{k}{2}y^2$, $k > 0$.
  SGD with step $\eta > 1/k$: $y_{t+1} = (1-\eta k)y_t$.
  Let $\rho = |1-\eta k| \in (0,1)$ for $\eta k \in (1,2)$.
  Then $\|g_t\| = k|y_t| = k\rho^t|y_0|$.
  The gradient norm decreases geometrically: $\|g_{t+1}\| < \|g_t\|$
  at every step, including oscillating steps where
  $g_t g_{t-1} < 0$.
  A gradient norm threshold detector $\|g_t\| > \tau$ for any
  fixed $\tau$ either fires on all steps or none—providing no
  oscillation-specific signal.
  Consecutive cosine: $\mathrm{cc}_t = -1$ at every step
  (exact sign flip), providing a perfect oscillation signal. \qed
\end{proof}

\end{document}